%% file: main.tex
\pdfoutput=1
\documentclass[runningheads]{llncs}

\usepackage[utf8]{inputenc}
\usepackage[T1]{fontenc}

\usepackage{graphicx}
\usepackage{amsmath,amssymb}
\usepackage{hyperref}

\usepackage{booktabs}
\usepackage{acronym}
\usepackage{csquotes}
\usepackage{tikz}
\usetikzlibrary{arrows,arrows.meta,automata,backgrounds,calc,patterns,positioning,shapes,shadows}
\usepackage[linesnumbered, ruled, noend]{algorithm2e}
\SetKw{Continue}{continue}
\usepackage[capitalise]{cleveref}
\usepackage{acronym}
\usepackage{subcaption}

\input{defs.tex}

\title{Estimating Causal Effects in Partially Directed Parametric Causal Factor Graphs}
\titlerunning{Estimating Causal Effects in PPCFGs}

\author{
	Malte Luttermann\inst{1} \and
	Tanya Braun\inst{2} \and
	Ralf Möller\inst{3} \and
	Marcel Gehrke\inst{3}
}
\institute{
	German Research Center for Artificial Intelligence (DFKI), Lübeck \\
	\email{malte.luttermann@dfki.de}
	\and
	Data Science Group, University of Münster, Germany \\
	\email{tanya.braun@uni-muenster.de}
	\and
	Institute for Humanities-Centered Artificial Intelligence, University of Hamburg \\
	\email{\{ralf.moeller,marcel.gehrke\}@uni-hamburg.de}
}

\begin{document}

\maketitle

\begin{abstract}
	Lifting uses a representative of indistinguishable individuals to exploit symmetries in probabilistic relational models, denoted as \aclp{pfg}, to speed up inference while maintaining exact answers.
	In this paper, we show how lifting can be applied to causal inference in partially directed graphs, i.e., graphs that contain both directed and undirected edges to represent causal relationships between \aclp{rv}.
	We present \acp{ppcfg} as a generalisation of previously introduced \aclp{pcfg}, which require a fully directed graph.
	% \Acp{ppcfg} no longer require a fully directed graphical model and hence provide a lifted representation that allows both for directed and undirected edges.
	We further show how causal inference can be performed on a lifted level in \acp{ppcfg}, thereby extending the applicability of lifted causal inference to a broader range of models requiring less prior knowledge about causal relationships.
\end{abstract}

\acresetall
\setcounter{footnote}{0}

\section{Introduction}
A fundamental problem for an intelligent agent performing reasoning under uncertainty is to compute the effect of an action on a certain \ac{rv} on other \acp{rv}.
When computing the effect of an action on a specific \ac{rv}, it is crucial to deploy the semantics of an intervention instead of performing a classical conditioning on that \ac{rv}~\cite[Chapter 4]{Pearl2009a}.
An intervention acting on a \ac{rv} $R$ can be thought of as setting $R$ to a fixed value and removing all incoming influences on the value of $R$.
In practice, generally not all causal relationships in a given model are known and thus, only a partially directed graphical model is available.
In such a partially directed graph, directed edges represent cause-effect relationships and undirected edges represent causal relationships whose direction is unknown.
In this paper, we solve the problem of efficiently estimating causal effects of actions in partially directed lifted probabilistic models, denoted as \aclp{pfg}.
Lifted representations are not only more expressive than propositional models such as \aclp{fg}, but also allow for tractable probabilistic inference with respect to domain sizes of \acp{lv} by exploiting symmetries.

\paragraph{Previous Work.}
The estimation of causal effects using causal graphical models in form of directed acyclic graphs in combination with observational data has been extensively studied in the literature (see, e.g., \cite{Pearl2009a,Pearl2016a,Spirtes2000a}).
Some works incorporate causal knowledge into (propositional) \aclp{fg} (which are originally undirected graphical models) to enable the estimation of causal effects in \aclp{fg}~\cite{Frey2003a,Winn2012a}.
In practice, the underlying causal graph is often not fully known and hence, identifying and estimating causal effects when provided with observational data and a partially directed graph has been investigated~\cite{Guo2021a,Liu2020a,Maathuis2009a,Perkovic2020a}.
However, all of these works perform causal effect estimation on a propositional level and thus lack the expressivity of relational logic, for example to capture the relationships between individual objects.
To represent individual objects and the relationships between them, Poole~\cite{Poole2003a} introduces \aclp{pfg} as lifted representations, which combine relational logic and probabilistic models, thereby allowing to encode that certain properties hold for groups of indistinguishable objects.
In probabilistic inference, lifting exploits symmetries to speed up inference while maintaining exact answers~\cite{Niepert2014a}.
Over the past years, both algorithms for symmetry detection~\cite{Ahmadi2013a,Kersting2009a,Luttermann2024a,Luttermann2024f,Luttermann2024d,Luttermann2023b} allowing the construction of lifted representations such as \aclp{pfg} as well as various lifted inference algorithms operating on \aclp{pfg} have been developed and further refined~\cite{Braun2016a,Braun2018a,DeSalvoBraz2005a,DeSalvoBraz2006a,Kisynski2009a,Milch2008a,Poole2003a,Taghipour2013a}.
More recently, Luttermann et al.~\cite{Luttermann2024b} introduce \acp{pcfg} as an extension of \aclp{pfg} allowing to incorporate causal knowledge into a lifted representation. % and build on the previously introduced lifted inference algorithms to enable the estimation of causal effects on a lifted level.
Nevertheless, the authors assume that the causal relationships between the involved \acp{rv} are fully known, which is rarely the case in practical settings.

\paragraph{Our Contributions.}
We introduce \acp{ppcfg} as a generalisation of \acp{pcfg} to obtain a formalism that compactly encodes a full joint distribution over a set of \acp{rv} and at the same time incorporates causal knowledge in the model, if available.
The major advantage of a \ac{ppcfg} over an \ac{pcfg} is that not all causal relationships between the involved \acp{rv} need to be known, thereby reducing the amount of prior knowledge required and thus making the model more suitable for many practical settings. % because often not all causal relationships are known in practice.
We further define $d$-separation in \acp{ppcfg} to reason about conditional independence.
In addition to that, we present an algorithm to efficiently estimate causal effects in \acp{ppcfg} on a lifted level, i.e., a representative of indistinguishable objects is used for computations to speed up inference.
Our algorithm identifies whether a causal effect can be uniquely determined from the given \ac{ppcfg} and if so, outputs the causal effect.
If the undirected edges in the \ac{ppcfg} lead to a causal effect being not uniquely identifiable, our algorithm efficiently enumerates all possible causal effects while operating on a lifted level.

\paragraph{Structure of this Paper.}
% The remainder of this paper is structured as follows.
We begin by introducing necessary background information and notations. %, starting with the introduction of \acp{pcfg} as a lifted representation of a fully directed graphical model and then continuing with the definition of an intervention.
Afterwards, we present \acp{ppcfg} as a generalisation of \acp{pcfg}, allowing both for directed and undirected edges in the model and then define $d$-separation in \acp{ppcfg}.
Thereafter, we provide an algorithm to efficiently estimate causal effects in \acp{ppcfg} before we conclude.

\section{Preliminaries} \label{sec:syn_prelims}
We begin to introduce \acp{prv}, which use \acp{lv} as parameters to represent sets of indistinguishable \acp{rv}.
\begin{definition}[Parameterised Random Variable]
	Let $\boldsymbol{R}$ be a set of \ac{rv} names, $\boldsymbol{L}$ a set of \ac{lv} names, and $\boldsymbol{D}$ a set of constants.
	All sets are finite.
	Each \ac{lv} $L$ has a domain $\domain{L} \subseteq \boldsymbol{D}$.
	A \emph{constraint} is a tuple $(\mathcal{X}, C_{\mathcal{X}})$ of a sequence of \acp{lv} $\mathcal{X} = (X_1, \ldots, X_n)$ and a set $C_{\mathcal{X}} \subseteq \times_{i = 1}^n\domain{X_i}$.
	The symbol $\top$ for $C$ marks that no restrictions apply, i.e., $C_{\mathcal{X}} = \times_{i = 1}^n\domain{X_i}$.
	A \emph{\ac{prv}} $R(L_1, \ldots, L_n)$, $n \geq 0$, is a syntactical construct of a \ac{rv} $R \in \boldsymbol{R}$ possibly combined with \acp{lv} $L_1, \ldots, L_n \in \boldsymbol{L}$ to represent a set of \acp{rv}.
	If $n = 0$, the \ac{prv} is parameterless and forms a propositional \ac{rv}.
	A \ac{prv} $A$ (or \ac{lv} $L$) under constraint $C$ is given by $A_{|C}$ ($L_{|C}$), respectively.
	We may omit $|\top$ in $A_{|\top}$ or $L_{|\top}$.
	The term $\range{A}$ denotes the possible values of a \ac{prv} $A$. 
	An \emph{event} $A = a$ denotes the occurrence of \ac{prv} $A$ with range value $a \in \range{A}$.
\end{definition}
\begin{example} \label{ex:prv_example}
	Consider $\boldsymbol{R} = \{Comp, Sal, Rev\}$ for competence, salary, and revenue, respectively, and $\boldsymbol{L} = \{E\}$ with $\domain{E} = \{alice, bob, charlie\}$ (employees), combined into \acp{prv} $Comp(E)$, $Sal(E)$, and $Rev$ with $\range{Comp(E)} = \range{Sal(E)} = \range{Rev} = \{\mathrm{low},\mathrm{medium},\mathrm{high}\}$.
\end{example}
A \ac{pf} describes a function, mapping argument values to positive real numbers, of which at least one is non-zero.
\begin{definition}[Parfactor]
	Let $\boldsymbol \Phi$ denote a set of factor names.
	We denote a \emph{\ac{pf}} $g$ by $\phi(\mathcal{A})_{| C}$ with $\mathcal{A} = (A_1, \ldots, A_n)$ being a sequence of \acp{prv}, $\phi \colon \times_{i = 1}^n \range{A_i} \mapsto \mathbb{R}^+$ being a function with name $\phi \in \boldsymbol \Phi$ mapping argument values to a positive real number called \emph{potential}, and $C$ being a constraint on the \acp{lv} of $\mathcal{A}$.
	We may omit $|\top$ in $\phi(\mathcal{A})_{|\top}$.
	The term $lv(Y)$ refers to the \acp{lv} in some element $Y$, a \ac{prv}, a \ac{pf}, or sets thereof.
	The term $gr(Y_{| C})$ denotes the set of all instances (groundings) of $Y$ with respect to constraint $C$.
\end{definition}
\begin{example} \label{ex:pf_example}
	Take a look at the \ac{pf} $g_1 = \phi_1(Comp(E), Rev)_{| \top}$.
	Assuming the same ranges of the \acp{prv} and the same domains of the \acp{lv} as in \cref{ex:prv_example}, $g_1$ specifies $\abs{\range{Comp(E)}} \cdot \abs{\range{Rev}} = 9$ input-output mappings $\phi_1(\mathrm{low}, \mathrm{low}) = \varphi_1$, $\phi_1(\mathrm{low}, \mathrm{medium}) = \varphi_2$, $\phi_1(\mathrm{low}, \mathrm{high}) = \varphi_3$, and so on with $\varphi_i \in \mathbb{R}^+$ for all $i = 1, \ldots, 9$.
	Further, $lv(g_1) = \{E\}$ and $gr(g_{1_{\top}}) = \{\phi_1(Comp(alice), Rev), \allowbreak \phi_1(Comp(bob), Rev), \allowbreak \phi_1(Comp(charlie), Rev)\}$.
\end{example}
A \ac{pcfg} is then built from a set of \acp{pf} $\{g_1, \ldots, g_m\}$ and the causal relationships between the \acp{prv}, which are encoded by the direction of the edges in the graph structure of the \ac{pcfg}~\cite{Luttermann2024b}.
In its original form, a \ac{pcfg} is a fully directed graph, but we extend this definition to allow for both directed and undirected edges in the following section.

\section{Partially Directed Parametric Causal Factor Graphs}
% So far, we have introduced \acp{prv} and \acp{pf}, which together form a parametric probabilistic model that combines probabilities and relational logic to allow for lifted probabilistic inference.
% First introduced by Poole~\cite{Poole2003a}, \aclp{pfg} combine probabilistic models and relational logic to represent sets of indistinguishable \acp{rv} in an undirected graph.
% More recently, \acp{pcfg} have been introduced to take advantage of the power of indistinguishability for the estimation of causal effects~\cite{Luttermann2024b}.
We now move on to define \acp{ppcfg} as lifted models that are able to incorporate partial causal knowledge, thereby allowing to exploit symmetries (in form of indistinguishable individuals) to speed up both probabilistic and causal inference by performing lifted inference.
% By combining relational logic with probabilistic models, a \ac{ppcfg} is more expressive than a propositional model such as a partially directed \acl{bn} or \acl{fg} and allows us to exploit symmetries (in form of indistinguishable individuals) to speed up both probabilistic and causal inference.
% as a generalisation of \acp{pcfg} (which require a fully directed model~\cite{Luttermann2024b}) to allow for both directed and undirected edges, thereby reducing the necessary amount of prior knowledge about the causal relationships in the model.
\begin{definition}[Partially Directed Parametric Causal Factor Graph] \label{def:pcfg}
	A \emph{\ac{ppcfg}} is a graph $M = (\boldsymbol A \cup \boldsymbol G, \boldsymbol E)$ that consists of variable nodes $\boldsymbol A$, factor nodes $\boldsymbol G$ ($\boldsymbol A \cap \boldsymbol G = \emptyset$), and a set of edges $\boldsymbol E$.
	Each variable node $A \in \boldsymbol A$ represents a \ac{prv} $A$ and every factor node $g \in \boldsymbol G$ represents a \ac{pf} $g = \phi(\mathcal{A})_{| C}$, where $\mathcal A = (A_1, \ldots, A_k)$ with $A_1 \in \boldsymbol A, \ldots, A_k \in \boldsymbol A$ is a sequence of \acp{prv}, $\phi \colon \times_{i = 1}^k \range{A_i} \mapsto \mathbb{R}^+$ is a function, and $C$ is a constraint on the \acp{lv} of $\mathcal{A}$.
	We again may omit $|\top$ in $\phi(\mathcal{A})_{|\top}$.
	For a variable node $A \in \boldsymbol A$ and a factor node $g \in \boldsymbol G$, there is an undirected edge $g - A \in \boldsymbol E$ if $A$ appears in the argument list of $g = \phi(\mathcal{A})$ and no information about the causal relationships between the \acp{prv} in $\mathcal A$ is available. % \{g, A\} instead of g - A and (g, A) for g \to A?
	If it is known that all $A' \in \mathcal A \setminus \{A\}$ are causes of $A \in \mathcal A$ (or if $\mathcal A \setminus \{A\} = \emptyset$), the edge $g - A$ can be replaced by a directed edge $g \to A$.

	The semantics of $M$ is given by grounding and building a full joint distribution.
	With $Z$ as the normalisation constant and $\mathcal A_k$ denoting the \acp{prv} appearing in the argument list of $\phi_k(\mathcal A_k)$, $M$ represents the full joint distribution
	$$
		P_M = \frac{1}{Z} \prod_{g \in \boldsymbol G} \prod_{\phi_k \in gr(g)} \phi_k(\mathcal A_k).
	$$
\end{definition}
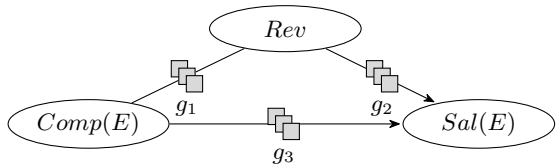
\begin{figure}[t]
	\centering
	\input{files/ppcfg_example.tex}
	\caption{A \ac{ppcfg} modelling the interplay of a company's revenue and its employees' competences and salaries (without input-output mappings of \acp{pf}).}
	\label{fig:ppcfg_example}
\end{figure}
\begin{example}
	Following \cref{ex:prv_example,ex:pf_example}, \cref{fig:ppcfg_example} depicts a \ac{ppcfg} $M$ modelling the interplay of a company's revenue and its employees' competences and salaries.
	Each \ac{pf} represents a group of ground factors and thus, grounding $M$ results in a partially directed \acl{fg} (see \cref{fig:ppcfg_ground_example}).
	In this particular example, the causal relationships encoded by $M$ tell us that the revenue of the company influences the salary of each individual employee and the competence of a specific employee influences their salary.
	Moreover, there is a dependency between the competence of each individual employee and the revenue of the company, but the causal direction is not encoded in $M$.
	We humans expect that the competence of the employees influences the revenue of the company, but an autonomous agent might not have this information available, resulting in a partially directed graph.
\end{example}
In accordance with the literature, in this paper we consider \acp{ppcfg} that do not contain directed cycles.
Throughout this paper, we denote the parents of a \ac{prv} $A$ in a \ac{ppcfg} $M = (\boldsymbol A \cup \boldsymbol G, \boldsymbol E)$ by $\Pa(A, M) = \{ A' \in \boldsymbol A \mid \exists g \in \boldsymbol G \colon (g - A') \in \boldsymbol E \land (g \to A) \in \boldsymbol E \}$.
Analogously, we define the children of $A$ in $M$ as $\Ch(A, M) = \{ A' \in \boldsymbol A \mid \exists g \in \boldsymbol G \colon (g \to A') \in \boldsymbol E \land (g - A) \in \boldsymbol E \}$ and the neighbours of $A$ in $M$ as $\Ne(A, M) = \{ A' \in \boldsymbol A \mid \exists g \in \boldsymbol G \colon (g - A') \in \boldsymbol E \land (g - A) \in \boldsymbol E \}$.

Note that the semantics of a \ac{ppcfg} is defined with respect to the set of \acp{pf} and thus is independent of the given causal relationships.
A \ac{ppcfg} hence simultaneously encodes a full joint probability distribution over a set of \acp{rv} (represented by \acp{prv}) and a set of causal relationships between those \acp{rv}\footnote{Thus, the setting of having an \ac{ppcfg} is equivalent to having a causal graph and observational data for the \acp{rv} occurring in the causal graph~\cite{Pearl2009a}.}.
Another important remark is that grounding a \ac{ppcfg} results in a partially directed \acl{fg} as introduced by Frey~\cite{Frey2003a}.
The important advantage is that the \ac{ppcfg} combines relational logic with probabilistic models, thereby being more expressive than a \acl{fg} and allowing us to exploit symmetries (in form of indistinguishable individuals) to speed up inference.
\begin{figure}[t]
	\centering
	\input{files/ppcfg_ground_example.tex}
	\caption{A visualisation of the resulting model when grounding the \ac{ppcfg} $M$ given in \cref{fig:ppcfg_example}, where $\domain{E} = \{alice,bob,charlie\}$.}
	\label{fig:ppcfg_ground_example}
\end{figure}
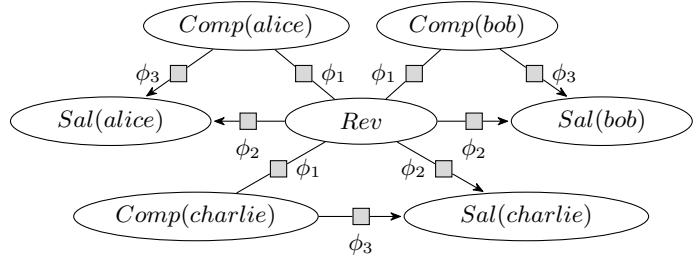
\begin{example}
	Take a look at \cref{fig:ppcfg_ground_example}, which presents the resulting model when grounding the \ac{ppcfg} $M$ given in \cref{fig:ppcfg_example}.
	The model $gr(M)$ now contains the \acp{rv} $Comp(alice)$, $Comp(bob)$, $Comp(charlie)$, which are represented by a single \ac{prv} $Comp(E)$ with a \ac{lv} $E$, $\domain{E} = \{alice,bob,charlie\}$, in $M$ (analogously for $Sal(alice)$, $Sal(bob)$, $Sal(charlie)$ and $Sal(E)$).
	Moreover, each \ac{pf} $g_i$ in $M$ represents a set of three (identical) ground factors $\phi_i$ in the grounded model $gr(M)$.
	In consequence, the lifted representation $M$ is more compact than its grounded counterpart $gr(M)$, thereby reducing run times required to perform inference.
	The underlying assumption is that there are indistinguishable objects, here employees, which can be represented by a representative.
\end{example}
Since a \ac{ppcfg} encodes a full joint probability distribution, a \ac{ppcfg} can be used to compute marginal distributions of grounded \acp{prv} given observations of specific events.
A query asks for a probability distribution (or a specific probability) given a set of observed events or, in the case of an interventional query, given a set of interventions fixing the values of certain (grounded) \acp{prv}.
\begin{definition}[Query]
	A \emph{query} $P(Q \mid E_1 = e_1, \ldots, E_k = e_k)$ consists of a query term $Q$ and a set of events $\{E_j = e_j\}_{j=1}^{k}$ with $Q$ and $E_j$ being grounded or parameterless \acp{prv}.
	To query a specific probability instead of a probability distribution, the query term is an event $Q = q$.
	An \emph{interventional query} $P(Q \mid do(R_1 = r_1, \ldots, R_k = r_k))$ asks for a probability distribution given that the grounded or parameterless \acp{prv} $R_j$ are set to a fixed value $r_j$.
\end{definition}
\begin{example}
	The query $P(Rev \mid Comp(alice) = \mathrm{high})$ asks for the probability distribution of $Rev$ given that the event $Comp(alice) = \mathrm{high}$ is observed.
	Keep in mind that, in general, observing and intervening are two different things, e.g., generally $P(Rev \mid Comp(alice) = \mathrm{high}) \neq P(Rev \mid do(Comp(alice) = \mathrm{high}))$.
\end{example}
Before we consider interventions in detail, we first take a closer look at conditional independence in \acp{ppcfg}, which is needed to reason about interventions.
% Afterwards, we formally define the semantics of an intervention and show how causal effects, which rely on the notion of an intervention, can efficiently be estimated in \acp{ppcfg}.

\section{Conditional Independence in PPCFGs}
The notion of $d$-separation~\cite{Pearl1986a} provides a graphical criterion to test for conditional independence in directed acyclic graphs and is essential to compute the effect of an intervention in the sense that all non-causal paths, so-called backdoor paths, need to be blocked to remove spurious effects.
Frey~\cite{Frey2003a} extends the notion of $d$-separation to partially directed \aclp{fg} and we build on this definition to define $d$-separation in \acp{ppcfg} (analogously to $d$-separation in \acp{pcfg}~\cite{Luttermann2024b}).
Note that the semantics of $d$-separation in \acp{ppcfg} is defined on a ground level.
\begin{definition}[$\boldsymbol d$-separation] \label{def:d_sep_ppcfg}
	Let $M = (\boldsymbol A \cup \boldsymbol G, \boldsymbol E)$ be a \ac{ppcfg}.
	Given three disjoint sets of \acp{rv} $\boldsymbol X$, $\boldsymbol Y$, and $\boldsymbol Z$ (subsets of $gr(\boldsymbol A)$), we say that $\boldsymbol X$ and $\boldsymbol Y$ are conditionally independent given $\boldsymbol Z$, written as $\boldsymbol X \upmodels \boldsymbol Y \mid \boldsymbol Z$, if the nodes in $\boldsymbol Z$ block all paths between the nodes in $\boldsymbol X$ and the nodes in $\boldsymbol Y$ in $gr(M)$.
	A path is a connected sequence of edges (independent of their directions) and it is therefore also possible for a path to pass from a parent of a factor to another parent of that factor.
	A path is blocked by the nodes in $\boldsymbol Z$ if
	\begin{enumerate}
		\item \label{item:d_sep_ppcfg_collider} the path contains the pattern $\phi_1 \to A \gets \phi_2$ such that neither $A$ nor any of its descendants are in $\boldsymbol Z$, or
		\item the path passes from $\phi_1$ through $A$ to $\phi_2$ such that it does not contain the pattern $\phi_1 \to A \gets \phi_2$ and $A$ is in $\boldsymbol Z$, or
		\item the path passes from a parent of a factor $\phi$ to another parent of $\phi$, and neither the child of $\phi$ nor any of its descendants are in $\boldsymbol Z$.
	\end{enumerate}
\end{definition}
\begin{example}
	Consider the grounded \ac{ppcfg} $M$ depicted in \cref{fig:ppcfg_ground_example}.
	$M$ encodes, for example, that the competence of $alice$ is independent of $bob$'s salary given the revenue of the company, written as $\{Comp(alice)\} \upmodels \{Sal(bob)\} \mid \{Rev\}$, because all paths from $Comp(alice)$ to $Sal(bob)$ pass through $Rev$.
\end{example}
\begin{remark}
	\Cref{def:d_sep_ppcfg} is slightly different from the definition of $d$-separation provided by Frey~\cite{Frey2003a} for partially directed \aclp{fg} in the sense that \cref{def:d_sep_ppcfg} is more flexible by allowing for variable nodes to have multiple parent factor nodes (e.g., in \cref{fig:ppcfg_example}, $Sal(E)$ has two parent factor nodes $g_2$ and $g_3$ instead of a single parent factor node with two inputs from $Comp(E)$ and $Rev$).
\end{remark}
\begin{remark} \label{remark:ci_in_ppcfg}
	In case a \ac{ppcfg} $M$ is constructed by hand (instead of being learned from observational data), it is possible to construct a mismatch of the conditional independence statements encoded in the graph structure of $M$ and the conditional independence statements implied by the full joint probability distribution encoded by $M$.
	We therefore assume that the graph structure of a given \ac{ppcfg} $M$ encodes exactly the same conditional independence statements as the underlying full joint probability distribution encoded by $M$.
\end{remark}
As we have seen, the semantics of $d$-separation in \acp{ppcfg} is defined on a ground level.
However, it is possible to check for $d$-separation on a lifted level without having to ground the entire \ac{ppcfg}.
For this purpose, the well-known Bayes-Ball algorithm~\cite{Shachter1998a} that allows us to check for conditional independence can be applied on a lifted level, that is, it can be run on the \ac{ppcfg} instead of on its grounded counterpart~\cite{Meert2010a}.
The idea of the Bayes-Ball algorithm is that a bouncing ball is sent through the graph structure such that the ball can pass through a node, bounce back, or be blocked to determine blocked paths (more details can be found in \cite{Shachter1998a}).
It is also possible to check whether \acp{prv} (instead of ground \acp{rv}) are conditionally independent in a highly efficient manner on a lifted level.
In a \ac{ppcfg}, each \ac{prv} $A$ is represented by a single variable node and thus, checking for conditional independence statements that involve $A$ can be done by looking at this single variable node instead of taking into account all groundings $gr(A)$ of $A$.
In contrast, in a propositional setting (i.e., in a ground model), each ground \ac{rv} in $gr(A)$ must be looked at individually (for example, to check whether $\{Comp(E)\} \upmodels \{Sal(E)\} \mid \{Rev\}$ holds, only three variable nodes are of relevance in the lifted representation while $2 \cdot \abs{\domain{E}} + 1$ nodes are of relevance in the ground model).

We next show how the estimation of causal effects, which rely on the notion of an intervention, can efficiently be realised in \acp{ppcfg}.
The idea is again to avoid grounding the entire \ac{ppcfg}, if possible. % and to perform the causal effect estimation on a lifted level to speed up causal inference.

\section{Efficient Estimation of Causal Effects in PPCFGs}
% A fundamental task when performing reasoning under uncertainty is to compute the effect of actions carried out on \acp{rv} on other \acp{rv}.
% The effect of an action can be measured by answering interventional queries~\cite{Pearl2009a}.
To compute the effect of actions carried out on \acp{rv} on other \acp{rv}, we have to answer interventional queries~\cite{Pearl2009a}.
An intervention on a grounded or parameterless \ac{prv} $R$, denoted as $do(R = r)$ where $r \in \range{R}$, changes the structure of the underlying model by setting the value of $R$ to $r$ and removing all incoming influences on $R$.
An intervention is defined on a fully directed graph.
\begin{definition}[Intervention] \label{def:intervention}
	Let $M = (\boldsymbol A \cup \boldsymbol G, \boldsymbol E)$ be a fully directed \ac{ppcfg} and let $gr(\boldsymbol A) = \{R_1, \ldots, R_n\}$ denote the set of \acp{rv} obtained by grounding the \acp{prv} in $\boldsymbol A$.
	Any probability distribution entailing the conditional independence statements encoded by $M$ can be factorised as $P(R_1, \ldots, R_n) = \prod_{i = 1}^n P(R_i \mid \Pa(R_i, gr(M)))$.
	An intervention $do(R'_1 = r'_1, \ldots, R'_k = r'_k)$ changes the underlying probability distribution such that
	\begin{align*}
		&P(R_1 = r_1, \ldots, R_n = r_n \mid do(R'_1 = r'_1, \ldots, R'_k = r'_k)) \\
		=
		&\begin{cases}
			\prod\limits_{R_i \in \{R_1, \ldots, R_n\} \setminus \{R'_1, \ldots, R'_k\}} P(r_i \mid \pa(R_i, gr(M))) & \text{if } \forall j \in \{1, \ldots, k\}\colon r_j = r'_j \\
			0 & \text{otherwise},
		\end{cases}
	\end{align*}
	where $\pa(R_i, gr(M))$ denotes the values of $\Pa(R_i, gr(M))$.
\end{definition}
Observe that the definition of an intervention refers to a fully directed \ac{ppcfg}.
In general, we deal with \acp{ppcfg} that might contain undirected edges and thus represent a whole class of fully directed \acp{ppcfg}, namely all fully directed \acp{ppcfg} that entail the same conditional independence statements as the initial \ac{ppcfg}\footnote{For classical (propositional) directed acyclic graphs, the set of fully directed acyclic graphs representing identical conditional independence statements is known under the name of a \emph{Markov equivalence class}.}.
% \begin{example} % $g_1 \to Comp(E)$ introduces a new v-structure?
% 	Consider again the \ac{ppcfg} $M$ given in \cref{fig:ppcfg_example}.
% 	$M$ represents two fully directed \acp{ppcfg} $M_1$, where the edge $g_1 - Comp(E)$ is replaced by $g_1 \to Comp(E)$, and $M_2$, where the edge $g_1 - Rev$ is replaced by $g_1 \to Rev$.
% 	Note that both $M_1$ and $M_2$ could possibly model the correct underlying causal relationships but we do not know whether $M_1$ or $M_2$ is actually the correct model.
% \end{example}
To determine the semantics of an intervention, we thus need to know the parents of the \acp{prv} on which we intervene.
As we deal with \acp{ppcfg} that might contain undirected edges, however, we do not always know the real parents of a \ac{prv}.
\begin{example}
	Given the \ac{ppcfg} depicted in \cref{fig:ppcfg_example}, the company might wonder whether it is worth it to send its employee $alice$ to an expensive training course to increase the revenue of the company due to the increased competence of $alice$.
	Hence, the company is interested in computing the quantity $P(Rev \mid do(Comp(alice) = \mathrm{high}))$.
	Since an intervention can be thought of as setting the value of $Comp(alice)$ to $\mathrm{high}$ and removing all incoming edges of $Comp(alice)$ in the graph, there are two possible scenarios when computing $P(Rev \mid do(Comp(alice) = \mathrm{high}))$: (i) $Rev$ is a parent of $Comp(alice)$ in the true model and thus, the underlying probability distribution is changed, or (ii) $Comp(alice)$ is a parent of $Rev$ in the true model and therefore, the underlying probability distribution remains unchanged as $Comp(alice)$ has no parents.
	Without further background information, we do not know which of these two scenarios is actually correct and thus, we need to consider both possibilities.
\end{example}
Fortunately, we do not always have to consider all possible fully directed \acp{ppcfg} represented by a given \ac{ppcfg} when computing the effect of an intervention because there are settings in which the represented fully directed \acp{ppcfg} all yield the same effect of the intervention.
In particular, in case all parents of the \acp{rv} on which we intervene are known, we can uniquely determine the effect of the intervention even if there are still undirected edges in the \ac{ppcfg}\footnote{In fact, this has been shown for propositional models~\cite{Maathuis2009a,Nandy2017a} and we now transfer this result to relational models.}.
\begin{theorem} \label{th:known_parents}
	Let $M = (\boldsymbol A \cup \boldsymbol G, \boldsymbol E)$ denote a \ac{ppcfg} and let $P(Q \mid do(R'_1 = r'_1, \ldots, R'_k = r'_k))$ be an interventional query with $Q \in gr(\boldsymbol A)$, $R'_1 \in gr(\boldsymbol A), \allowbreak \ldots, \allowbreak R'_k \in gr(\boldsymbol A)$, and $\{Q\} \cap \{R'_1, \ldots, R'_k\} = \emptyset$.
	If it holds that $\Ne(R'_1, gr(M)) = \emptyset, \ldots, \Ne(R'_k, gr(M)) = \emptyset$, then the result of $P(Q \mid do(R'_1 = r'_1, \ldots, R'_k = r'_k))$ is identical for all fully directed \acp{ppcfg} represented by $M$.
\end{theorem}
\begin{proof}
	Let $\{R_1, \ldots, R_{\ell}\} = gr(\boldsymbol A) \setminus \{Q, R'_1, \ldots, R'_k\}$ denote the set of all grounded \acp{rv} that do not occur in the given query.
	From \cref{def:intervention} we know that
	\begin{align} \label{eq:intervention_distribution}
		\begin{split}
			&P(Q = q \mid do(R'_1 = r'_1, \ldots, R'_k = r'_k)) \\
			=
			&
			\sum\limits_{r_1 \in \range{R_1}} \ldots \sum\limits_{r_{\ell} \in \range{R_{\ell}}} \prod\limits_{R_i \in gr(\boldsymbol A) \setminus \{R'_1, \ldots, R'_k\}} P(r_i \mid \pa(R_i, gr(M))).
		\end{split}
	\end{align}
	If it holds that $\Ne(R'_1, gr(M)) = \emptyset, \ldots, \Ne(R'_k, gr(M)) = \emptyset$, then the parents of $R'_1, \ldots, R'_k$ are fully known and identical in all fully directed \acp{ppcfg} represented by $M$.
	Hence, it remains to be shown that the factorisation of all ground \acp{rv} that are not in $\{R'_1, \ldots, R'_k\}$ is equivalent for all fully directed \acp{ppcfg} represented by $M$.
	We know that all fully directed \acp{ppcfg} represented by $M$ entail exactly the same conditional independence statements as $M$ and thus, the factorisations induced by these fully directed \acp{ppcfg} entail equivalent semantics (just as all \acl{bn} structures over  a fixed set of \acp{rv} entailing the same conditional independence statements induce equivalent factorisations of the underlying probability distribution).
	Consequently, \cref{eq:intervention_distribution} yields the same result for all fully directed \acp{ppcfg} represented by $M$.
\end{proof}
\Cref{th:known_parents} implies that we do not have to consider all possible edge directions of the undirected edges in a \ac{ppcfg} when computing the effect of an intervention but just the possible directions of the undirected edges that are relevant for the intervention, that is, the directions of the undirected edges that are connected to the \acp{rv} on which we intervene.
Note that all terms required to answer the interventional query can be computed by querying the \ac{ppcfg}, as the \ac{ppcfg} compactly encodes the full joint probability distribution over all ground \acp{rv}.
The semantics of the \ac{ppcfg} is well-defined even if there are undirected edges in the graph because the factors do not necessarily encode distributions conditioned on the parents but instead encode arbitrary local distributions that factorise the full joint probability distribution.

Intuitively, it becomes clear that the effect of an intervention cannot be uniquely determined if there are undirected edges connected to the \acp{rv} on which we intervene because there are different possible parent sets that might result in multiple disjoint effects of the intervention.
%We next show that this result does not only hold for propositional models~\cite{Maathuis2009a} but also for \acp{ppcfg}.
\begin{theorem} \label{th:unknown_parents}
	Let $M = (\boldsymbol A \cup \boldsymbol G, \boldsymbol E)$ denote a \ac{ppcfg} and let $P(Q \mid do(R'_1 = r'_1, \ldots, R'_k = r'_k))$ be an interventional query with $Q \in gr(\boldsymbol A)$, $R'_1 \in gr(\boldsymbol A), \allowbreak \ldots, \allowbreak R'_k \in gr(\boldsymbol A)$, and $\{Q\} \cap \{R'_1, \ldots, R'_k\} = \emptyset$.
	If there exists a \ac{rv} $R'_i \in \{R'_1, \ldots, R'_k\}$ such that $\Ne(R'_i, gr(M)) \neq \emptyset$, then the result of $P(Q \mid do(R'_1 = r'_1, \ldots, R'_k = r'_k))$ is not guaranteed to be uniquely determined.
	% If given model is an MPDAG (i.e., Meek rules have been applied extensively), then the result is guaranteed to be not uniquely determined.
\end{theorem}
\begin{proof}
	If there exists a \ac{rv} $R'_i \in \{R'_1, \ldots, R'_k\}$ such that $\Ne(R'_i, gr(M)) \neq \emptyset$ holds, there are undirected edges connected to $R'_i$, implying that the parents of $R'_i$ are not guaranteed to be identical in all fully directed \acp{ppcfg} represented by $M$.
	Since we know that \cref{eq:intervention_distribution} gives us the result of the given query, the factors being removed from the product on the right hand side of the equation might differ depending on the actual parents of the ground \acp{rv} on which we intervene, thereby possibly yielding different results for the query.
\end{proof}
\begin{remark}
	There are scenarios in which it is possible to uniquely determine the result of an interventional query even if there are undirected edges connected to the \acp{rv} on which we intervene as not all undirected edges can be oriented in both directions (because they are not allowed to introduce a cycle or to change the conditional independence statements encoded in the graph structure)\footnote{An orientation of an undirected edge alters the conditional independence statements if the new orientation introduces the pattern $R_1 - \phi_1 \to R_2 \gets \phi_2 - R_3$ such that $R_1$ and $R_3$ are not directly connected via a factor---see \cref{item:d_sep_ppcfg_collider} in \cref{def:d_sep_ppcfg}.}.
	% Consider a scenario with an intervention $do(R' = r')$ and a neighbour $R \in \Ne(R', gr(M))$ of $R'$, connected to $R'$ via a factor $\phi$ over edges $\phi - R$ and $\phi - R'$.
	% Even though $R - \phi - R'_i$ might be oriented as $R - \phi \to R'_i$ or as $R \gets \phi - R'_i$ in general, it is not possible to have an orientation that introduces a new cycle or an orientation that changes the conditional independence statements encoded by the given model (which is the case if the new orientation introduces the pattern $R_1 - \phi_1 \to R_2 \gets \phi_2 - R_3$ such that $R_1$ and $R_3$ are not directly connected via a factor because $R_1$ and $R_3$ then become unconditionally independent although they are not---see \cref{item:d_sep_ppcfg_collider} in \cref{def:d_sep_ppcfg}).
	% In consequence, the result of an interventional query can be uniquely determined if the undirected edges connected to the ground \acp{rv} on which we intervene can only be oriented in a single direction.
\end{remark}
Combining the insights from \cref{th:known_parents,th:unknown_parents} naturally leads to an algorithm to compute the effect of interventions in \acp{ppcfg}.
The idea is that all possible parent sets of the intervention variables have to be considered.
If there is just one possible set of parents, the effect of the intervention can be uniquely determined, otherwise there are multiple possible effects that are enumerated.
This idea is incorporated in the IDA algorithm and its variants~\cite{Guo2021a,Liu2020a,Maathuis2009a} for single interventions (i.e., interventions $do(R' = r')$) in propositional models.
\Cref{alg:lci_ppcfg} displays our proposed algorithm, which extends the idea of just considering the possible parent sets of intervention variables to joint interventions (i.e., interventions $do(R'_1 = r'_1, \ldots, R'_k = r'_k)$ where $k \geq 1$) in relational models\footnote{\Cref{alg:lci_ppcfg} transfers the results for joint interventions in propositional models with undirected edges~\cite{Nandy2017a} to relational models.}.%, however, additional assumptions regarding the linearity of the underlying structural equations are postulated there.}. % i.e., they use a slightly different framework (structural equation models) compared to us.

\begin{algorithm}[t]
	\SetKwInOut{Input}{Input}
	\SetKwInOut{Output}{Output}
	\caption{Lifted Causal Inference in \acp{ppcfg}}
	\label{alg:lci_ppcfg}
	\Input{A \ac{ppcfg} $M = (\boldsymbol A \cup \boldsymbol G, \boldsymbol E)$ and an interventional query $P(Q \mid do(R'_1 = r'_1, \ldots, R'_k = r'_k))$ with $Q \in gr(\boldsymbol A)$, $R'_1 \in gr(\boldsymbol A), \ldots, R'_k \in gr(\boldsymbol A)$ and $\{Q\} \cap \{R'_1, \ldots, R'_k\} = \emptyset$.}
	\Output{The set of all possible post-intervention distributions $P(Q \mid do(R'_1 = r'_1, \ldots, R'_k = r'_k))$ represented by $M$.}
	\BlankLine
	$\boldsymbol P \gets \emptyset$\;
	$\{R_1, \ldots, R_{\ell}\} \gets gr(\boldsymbol A) \setminus \{Q, R'_1, \ldots, R'_k\}$\;
	$M' \gets$ \Ac{pcfg} after splitting \acp{pf} in $M$ based on each $R'_i \in \{R'_1, \ldots, R'_k\}$\;
	\ForEach{$\boldsymbol C_1 \subseteq \Ne(R'_1, M'), \ldots, \boldsymbol C_k \subseteq \Ne(R'_k, M')$ s.t.\ $\boldsymbol C_1, \ldots, \boldsymbol C_k$ are cliques}{ \label{line:lci_ppcfg_loop} % such that $\boldsymbol C_1, \ldots, \boldsymbol C_k$ are cliques
		$M'' \gets$ \Ac{pcfg} after orienting all edges $C - \phi - R'_i$ as $C - \phi \to R'_i$ in $M'$ for all $C \in \boldsymbol C_i$ for every $i \in \{1, \ldots, k\}$\; \label{line:lci_ppcfg_orient_parents}
		$M'' \gets$ Arbitrary fully directed \ac{ppcfg} represented by $M''$\;
		\If{no such $M''$ exists}{
			\Continue\;
		}
		$D \gets \sum\limits_{r_1 \in \range{R_1}} \ldots \sum\limits_{r_{\ell} \in \range{R_{\ell}}} \prod\limits_{R_i \in gr(\boldsymbol A) \setminus \{R'_1, \ldots, R'_k\}} P(r_i \mid \pa(R_i, M''))$\; \label{line:lci_ppcfg_formula} % Describe how parents of a ground \ac{rv} can be found directly in the lifted model without grounding
		% Also, exchangeable RVs can be queried once instead of every one separately, e.g., R(alice | ...), R(bob | ...) and R(charlie | ...) should be identical for those RVs on which we do not intervene
		Add $D$ to $\boldsymbol P$\;
	}
	\Return{$\boldsymbol P$}
\end{algorithm}

Given a \ac{ppcfg} $M = (\boldsymbol A \cup \boldsymbol G, \boldsymbol E)$ and an interventional query $P(Q \mid do(R'_1 = r'_1, \ldots, R'_k = r'_k))$, \cref{alg:lci_ppcfg} proceeds as follows to compute the set of all possible results for the given query.
First, \cref{alg:lci_ppcfg} splits the \acp{pf} in $M$ based on each $R'_i \in \{R'_1, \ldots, R'_k\}$.
Splitting the \acp{pf} results in a modified \ac{ppcfg} $M'$ entailing equivalent semantics as $M$ and works as follows~\cite{DeSalvoBraz2005a}.
Recall that $R'_i = A(L_1 = l_1, \dots, L_j = l_j)$, $l_1 \in \domain{L_1}, \dots, l_j \in \domain{L_j}$, is a specific instance of a \ac{prv} $A(L_1, \dots, L_j)$.
The idea behind the splitting procedure is that we would like to separate $gr(A)$ into two sets $gr(A) \setminus \{R'_i\}$ and $\{R'_i\}$, as $R'_i$ has to be treated differently than the remaining instances of $A$ due to the intervention.
Thus, every \ac{pf} $g$ for which there is an instance $\phi \in gr(g)$ such that $R'_i$ appears in the argument list of $\phi$ is split.
Formally, splitting a \ac{pf} $g$ replaces $g$ by two \acp{pf} $g'_{| C'}$ and $g''_{| C''}$ and adapts the constraints of $g'_{| C'}$ and $g''_{| C''}$ such that the inputs of $g'_{| C'}$ are restricted to all sequences that contain $R'_i$ and the inputs of $g''_{| C''}$ are restricted to the remaining input sequences. % The constraint is within g, so g = \phi(A)_{| C_1} would be the correct notation
After the splitting procedure, the semantics of the model remains unchanged as the groundings of $M'$ are still the same as the groundings of the initial model $M$---they are just arranged differently across the sets of ground instances.

\Cref{alg:lci_ppcfg} then iterates over all possible parent sets (i.e., over all subsets of undirected neighbours) of $R'_1, \ldots, R'_k$.
When considering the subsets of undirected neighbours, it is necessary that all subsets are jointly valid, that is, they are not allowed to alter the conditional independence statements encoded by the model and they must not introduce any cycles when oriented towards $R'_1, \ldots, R'_k$.
To ensure the validity of these subsets, they are required to form a clique.
A clique $\boldsymbol C$ is a subset of nodes such that all pairs of nodes in $\boldsymbol C$ are directly connected via a factor\footnote{Formally, a clique $\boldsymbol C$ in a \ac{ppcfg} $M = (\boldsymbol A \cup \boldsymbol G, \boldsymbol E)$ is a subset of nodes such that for each pair of nodes $C_1 \in \boldsymbol C$, $C_2 \in \boldsymbol C$ with $C_1 \neq C_2$ it holds that there exists a factor $\phi$ such that there is an edge between $C_1$ and $\phi$ as well as an edge between $C_2$ and $\phi$ in $\boldsymbol E$ (either directed or undirected).}.
By ensuring that the subsets of undirected neighbours form cliques, the orientation of the undirected neighbours of $R'_1, \ldots, R'_k$ towards $R'_1, \ldots, R'_k$ does not introduce any pattern $R_1 - \phi_1 \to R_2 \gets \phi_2 - R_3$ such that $R_1$ and $R_3$ are not directly connected via a factor, as due to the clique property $R_1$ and $R_3$ are always guaranteed to be directly connected via a factor.
Thus, the conditional independence statements encoded by $M''$ are equivalent to those encoded by $M$.
Having obtained the parent sets of $R'_1, \ldots, R'_k$, \cref{alg:lci_ppcfg} next extends the modified model $M''$ to any fully directed \ac{ppcfg} represented by $M''$, if such a fully directed \ac{ppcfg} exists.
Afterwards, $M''$ is guaranteed to be a fully directed \ac{ppcfg} represented by $M$ and hence, the result of the provided query is given by \cref{eq:intervention_distribution}, which requires us to know the parents of all non-intervention \acp{rv} to compute a product over conditional probability distributions.
Since $M$ encodes the full joint probability distribution over $gr(\boldsymbol A)$, the conditional probability distributions are obtained by querying $M$, which allows us to compute the results for these queries using lifted probabilistic inference, thereby avoiding to ground the entire model if possible.
The resulting post-intervention distribution is then added to the result set $\boldsymbol P$ and the algorithm continues the above steps for the next possible parent set until all possible parent sets have been taken into account.
% In case there is no causal explanation for the given \ac{ppcfg} $M$ (i.e., $M$ does not represent any fully directed \ac{ppcfg}), \cref{alg:lci_ppcfg} returns an empty set.
% Such a situation might occur if $M$ already contains a directed cycle or if there are undirected edges that cannot be oriented without introducing a new cycle or changing the conditional independence statements encoded by the model.
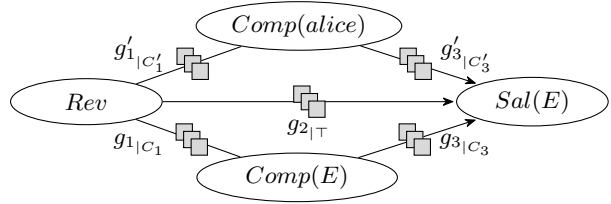
\begin{figure}[t]
	\centering
	\input{files/ppcfg_split_example.tex}
	\caption{The modified \ac{ppcfg} obtained after splitting the \acp{pf} in the \ac{ppcfg} shown in \cref{fig:ppcfg_example} to separate $Comp(alice)$ from $Comp(E)$.}
	\label{fig:ppcfg_split_example}
\end{figure}
\begin{example}
	Assume we want to compute $P(Rev \mid do(Comp(alice) = \mathrm{high}))$ in the \ac{ppcfg} $M$ depicted in \cref{fig:ppcfg_example}.
	As the intervention operates on $Comp(alice)$, first the \acp{pf} in $M$ are split on $Comp(alice)$ to obtain the modified \ac{ppcfg} $M'$ shown in \cref{fig:ppcfg_split_example}.
	In $M'$, $Comp(alice)$ is now a separate node in the graph, so its possible parents can be determined.
	Formally, $g_1$ has been replaced by two \acp{pf} $g_{1_{| C_1}}$ and $g'_{1_{| C'_1}}$ with constraints $C_1 = (E, \{bob,charlie\})$ and $C'_1 = (E, \{alice\})$.
	Analogously, $g_3$ has been replaced by $g_{3_{| C_3}}$ and $g'_{3_{| C'_3}}$.
	Next, \cref{alg:lci_ppcfg} considers all possible parent sets of $Comp(alice)$, which are given by $\Pa_1(Comp(alice), M'') = \emptyset$ (resulting from orienting $Comp(alice) - g'_{1_{| C'_1}} - Rev$ as $Comp(alice) - g'_{1_{| C'_1}} \to Rev$) and $\Pa_2(Comp(alice), M'') = \{Rev\}$ (resulting from orienting $Comp(alice) - g'_{1_{| C'_1}} - Rev$ as $Comp(alice) \gets g'_{1_{| C'_1}} - Rev$).
	% Both orientations are valid, i.e., they do not change the conditional independence statements encoded by the model and they do not introduce a cycle. % It is possible to assume that all employees must have the same edge directions but we do not make this assumption. With this assumption, $Comp(alice) - g'_{1_{| C'_1}} \to Rev$ becomes invalid as a new pattern of the form $R_1 - \phi_1 \to R_2 \gets \phi_2 - R_3$ such that $R_1$ and $R_3$ are not directly connected via a factor is introduced.
	Depending on whether $Comp(alice)$ actually has a parent or not, the post-intervention distribution is different and hence, \cref{alg:lci_ppcfg} returns a set containing two possible results for the query $P(Rev \mid do(Comp(alice) = \mathrm{high}))$.
\end{example}
Note that the model shown in \cref{fig:ppcfg_split_example} is more compact than the fully grounded model in \cref{fig:ppcfg_ground_example} as \cref{alg:lci_ppcfg} only grounds necessary parts of the model.
\Cref{alg:lci_ppcfg} can further benefit from interventions on multiple indistinguishable objects under the assumption that the graph structure is identical for all groundings.
For example, when considering the effect of a training course given to multiple employees on the revenue of the company, even if there are hundreds of employees, after intervening on their competence there are still only two nodes representing their competences in the graph, namely one node for all employees that have received the training course and another node for the remaining employees.
% Under this assumption the undirected edge has only one possible orientation as orienting Rev <- Comp(E) introduces a new v-structure
In this case, \cref{alg:lci_ppcfg} has to consider only two possible parent sets regardless of the number of employees while there are $2^{\abs{\domain{E}}}$ possible parent sets in an equivalent propositional model to consider (where $\abs{\domain{E}}$ is the number of employees). % In a propositional model, it is also possible to reduce the number of possible parent sets when background knowledge is introduced, i.e., when knowing that specific \acp{rv} are actually representable by a single \ac{prv}.
\begin{corollary}
	Let $M = (\boldsymbol A \cup \boldsymbol G, \boldsymbol E)$ be a \ac{ppcfg}.
	When intervening on a non-grounded \ac{prv} $A(L_1, \ldots, L_j) \in \boldsymbol A$, under the assumption that the graph structure is identical for all groundings, it holds that
	% that its neighbouring edge directions are identical for all of its groundings, it holds that
	\begin{enumerate}
		\item \cref{alg:lci_ppcfg} considers $O(2^{\abs{\Ne(A, M)}})$ possible parent sets in the worst case, and
		\item in a propositional model, $O(2^{\prod_{R \in gr(A)} \abs{\Ne(R, gr(M))}})$ possible parent sets have to be considered in the worst case. % assumption not necessary for this statement (just for the first one)
	\end{enumerate}
\end{corollary}
% It is also possible to intervene on a non-grounded \ac{prv}.
% For example, the company might wonder whether it is worth it to send all its employees to a training course to increase the revenue of the company.
% Then, we are interested in computing the quantity $P(Rev \mid do(Comp(E) = \mathrm{high}))$. %, which is semantically equivalent to $P(Rev \mid do(Comp(alice) = \mathrm{high}, Comp(bob) = \mathrm{high}, Comp(charlie) = \mathrm{high}))$.
Finally, note that \cref{alg:lci_ppcfg} can handle multiple query variables at once instead of a single query variable $Q$ (the query variables are then not summed out in \cref{line:lci_ppcfg_formula}).

\section{Conclusion}
We introduce \acp{ppcfg} as probabilistic relational models that allow to incorporate partial causal knowledge, thereby enabling lifted causal inference without the need for a fully specified causal model.
A lifted representation such as a \ac{ppcfg} is more expressive than a propositional model and allows for tractable inference with respect to domain sizes of \acp{lv}.
We further present an algorithm to efficiently compute the effect of joint interventions in \acp{ppcfg} (i.e., on a lifted level).
Our proposed algorithm is also able to efficiently deal with interventions on \acp{prv} representing sets of indistinguishable \acp{rv}.

In future work, we aim to investigate the effect of interventions in a relational model with mutual interdependencies in form of bidirectional edges.
We conjecture that \acp{rv} with mutual interdependencies can be collapsed into a single node in the graph such that our proposed algorithm can still be applied.
Another interesting direction for future work is to allow for hidden confounders.
% In future work, the computation of the effect on interventions on so-called \aclp{crv} is an interesting problem, which further extends the expressivity of \acp{ppcfg} by allowing to answer questions such as \enquote{What is the effect of sending any group of ten employees to a training course on the revenue of the company?} without knowing which specific employees are part of the group.
% Another interesting direction is to investigate the effect of interventions in models with hidden confounders.

\section*{Acknowledgements}
This work is funded by the BMBF project AnoMed 16KISA057.
This preprint has not undergone peer review or any post-submission improvements or corrections.
The Version of Record of this contribution is published in \emph{Lecture Notes in Computer Science, Volume 15350}, and is available online at \url{https://link.springer.com/chapter/10.1007/978-3-031-76235-2_20}.
% Cite DAG extension paper?

\bibliographystyle{splncs04}
\bibliography{references}

% \clearpage
% \appendix
% \acresetall

% \begin{theorem} % There are 2^n graphs that we have to consider, however, this is not a problem as the factor already encodes 2^n mappings that are considered anyway during inference -> complexity remains unchanged
% 	Let $M$ and $P(Q \mid do(R'_1 = r'_1, \ldots, R'_k = r'_k))$ denote the input \ac{ppcfg} and query for \cref{alg:lci_ppcfg}.
% 	Then, \cref{alg:lci_ppcfg} considers $O(2^d)$ possible parent sets in \cref{line:lci_ppcfg_loop} in the worst case, where $d = \abs{\Ne(R'_1, M)} + \ldots + \abs{\Ne(R'_k, M)}$.
% \end{theorem}
% \begin{proof}
% 	Sketch: Each $R'_i$ has $\abs{\Ne(R'_i, M)}$ undirected neighbours, so there are $2^{\abs{\Ne(R'_i, M)}}$ possible orientations induced by those neighbours (each edge can either be directed as an ingoing edge to $R'_i$ or as an outgoing edge from $R'_i$).
% 	Thus, there are
% 	\begin{align*}
% 		\prod_{i = 1}^k 2^{\abs{\Ne(R'_i, M)}} = 2^{\sum_{i = 1}^k \abs{\Ne(R'_i, M)}} = 2^d
% 	\end{align*}
% 	possible parent sets to consider at most.
% \end{proof}

\end{document}

%% file: defs.tex
\acrodef{acp}[ACP]{advanced colour passing}
\acrodef{bn}[BN]{Bayesian network}
\acrodef{cp}[CP]{colour passing}
\acrodef{crv}[CRV]{counting randvar}
\acrodef{decor}[DECOR]{detection of commutative factors}
\acrodef{deft}[DEFT]{detection of exchangeable factors}
\acrodef{er}[ER]{entity-relationship}
\acrodef{fg}[FG]{factor graph}
\acrodef{ljt}[LJT]{lifted junction tree}
\acrodef{lv}[logvar]{logical variable}
\acrodef{lve}[LVE]{lifted variable elimination}
\acrodef{mln}[MLN]{Markov logic network}
\acrodef{mn}[MN]{Markov network}
\acrodef{mpdag}[MPDAG]{maximally oriented partially directed acyclic graph}
\acrodef{pcfg}[PCFG]{parametric causal factor graph}
\acrodef{pcrv}[PCRV]{parameterised CRV}
\acrodef{ppcfg}[PPCFG]{partially directed parametric causal factor graph}
\acrodef{pf}[parfactor]{parametric factor}
\acrodef{pfg}[PFG]{parametric factor graph}
\acrodef{prv}[PRV]{parameterised randvar}
\acrodef{rv}[randvar]{random variable}
\acrodef{ve}[VE]{variable elimination}
\acrodef{wl}[WL]{Weisfeiler-Leman}

\crefname{algocf}{Alg.}{Algs.}
\Crefname{algocf}{Algorithm}{Algorithms}
\crefname{algorithm}{Alg.}{Algs.}
\Crefname{algorithm}{Algorithm}{Algorithms}
\crefname{corollary}{Cor.}{Cors.}
\Crefname{corollary}{Corollary}{Corollaries}
\crefname{definition}{Def.}{Defs.}
\Crefname{definition}{Definition}{Definitions}
\crefname{example}{Ex.}{Exs.}
\Crefname{example}{Example}{Examples}
\crefname{proposition}{Prop.}{Props.}
\Crefname{proposition}{Proposition}{Propositions}
\crefname{section}{Sec.}{Secs.}
\Crefname{section}{Section}{Sections}
\crefname{theorem}{Thm.}{Thms.}
\Crefname{theorem}{Theorem}{Theorems}

\newcommand{\abs}[1]{\lvert #1 \rvert}

\newcommand{\domain}[1]{\ensuremath{\mathrm{dom}(#1)}}

\newcommand{\range}[1]{\ensuremath{\mathrm{range}(#1)}}

\newcommand{\upmodels}{\ensuremath{\perp\!\!\!\perp}}

\newcommand{\Ch}{\mathrm{Ch}}
\newcommand{\Ne}{\mathrm{Ne}}
\newcommand{\Pa}{\mathrm{Pa}}
\newcommand{\pa}{\mathrm{pa}}

\definecolor{myyellow}{RGB}{247,192,26}
\definecolor{myblue}{RGB}{37,122,164}
\definecolor{mygreen}{RGB}{78,155,133}
\definecolor{mypurple}{RGB}{86,51,94}

\definecolor{newblue}{RGB}{50,113,173}
\definecolor{newred}{RGB}{222,32,36}
\definecolor{newgreen}{RGB}{70,165,69}
\definecolor{newpurple}{RGB}{140,69,152}

\definecolor{cborange}{RGB}{230,159,0}
\definecolor{cbblue}{RGB}{30,136,229}
\definecolor{cbbluedark}{RGB}{46,37,133}
\definecolor{cbpurple}{RGB}{170,68,153}
\definecolor{cbgreen}{RGB}{0,77,64}
\definecolor{cbgreenlight}{RGB}{93,168,153}
\definecolor{cbbrown}{RGB}{126,41,84}

\pgfdeclarelayer{bg}
\pgfsetlayers{bg,main}

\tikzset{
	rv/.style={draw, ellipse},
	pf/.style={draw, rectangle, fill = gray!30},
	arc/.style = {->, >={[round,sep]Stealth}},
}

\newcommand\factorat[4]{
	\node[pf, label={#2:{#3}}](#4) at (#1) {};
}

\newcommand\pfsat[6]{
	\node[pf, xshift=-1mm, yshift=1mm](#4) at (#1) {};
	\node[pf, label={[label distance=1mm]#2:{#3}}](#5) at (#1) {};
	\node[pf, xshift=1mm, yshift=-1mm](#6) at (#1) {};
}

\newcommand\pfsatd[7]{
	\node[pf, xshift=-1mm, yshift=1mm](#4) at (#1) {};
	\node[pf, label={[label distance=#7]#2:{#3}}](#5) at (#1) {};
	\node[pf, xshift=1mm, yshift=-1mm](#6) at (#1) {};
}

%% file: files/ppcfg_example.tex
\begin{tikzpicture}
	\node[rv, minimum width = 6.2em, draw] (rev) {$Rev$};
	\node[rv, minimum width = 6.2em, draw, below left  = 2.5em and 3.5em of rev, inner sep = 2.0pt] (comp) {$Comp(E)$};
	\node[rv, minimum width = 6.2em, draw, below right = 2.5em and 3.5em of rev, inner sep = 2.0pt] (sal) {$Sal(E)$};

	\pfsat{$(rev)!0.5!(comp)$}{270}{$g_1$}{f1a}{f1}{f1b}
	\pfsat{$(rev)!0.5!(sal)$}{270}{$g_2$}{f2a}{f2}{f2b}
	\pfsat{$(comp)!0.5!(sal)$}{270}{$g_3$}{f3a}{f3}{f3b}

	\begin{pgfonlayer}{bg}
		\draw (comp) -- (f1);
		\draw (f1) -- (rev);
		\draw (rev) -- (f2);
		\draw[arc] (f2) -- (sal);
		\draw (comp) -- (f3);
		\draw[arc] (f3) -- (sal);
	\end{pgfonlayer}
\end{tikzpicture}

%% file: files/ppcfg_ground_example.tex
\begin{tikzpicture}
	\node[rv, minimum width = 6.2em, draw] (rev) {$Rev$};
	\node[rv, minimum width = 8.2em, draw, left = 3em of rev, inner sep = 2.0pt] (salA) {$Sal(alice)$};
	\node[rv, minimum width = 7.4em, draw, right = 3em of rev, inner sep = 2.0pt] (salB) {$Sal(bob)$};
	\node[rv, minimum width = 8.2em, draw, above left   = 2.5em and -0.5em of rev, inner sep = 2.0pt] (compA) {$Comp(alice)$};
	\node[rv, minimum width = 7.4em, draw, above right  = 2.5em and -0.5em of rev, inner sep = 2.0pt] (compB) {$Comp(bob)$};
	\node[rv, minimum width = 10.0em, draw, below right = 2.5em and 1em of rev, inner sep = 2.0pt] (salC) {$Sal(charlie)$};
	\node[rv, minimum width = 10.0em, draw, below left  = 2.5em and 1em of rev, inner sep = 2.0pt] (compC) {$Comp(charlie)$};

	\factorat{$(rev)!0.5!(compA)$}{0}{$\phi_1$}{f11}
	\factorat{$(rev)!0.5!(compB)$}{180}{$\phi_1$}{f12}
	\factorat{$(rev)!0.5!(compC)$}{0}{$\phi_1$}{f13}
	\factorat{$(rev.west)!0.5!(salA.east)$}{270}{$\phi_2$}{f21}
	\factorat{$(rev.east)!0.5!(salB.west)$}{270}{$\phi_2$}{f22}
	\factorat{$(rev)!0.5!(salC)$}{180}{$\phi_2$}{f23}

	\factorat{$(compA)!0.5!(salA)$}{180}{$\phi_3$}{f31}
	\factorat{$(compB)!0.5!(salB)$}{0}{$\phi_3$}{f32}
	\factorat{$(compC)!0.5!(salC)$}{270}{$\phi_3$}{f33}

	\begin{pgfonlayer}{bg}
		\draw (rev) -- (f11);
		\draw (rev) -- (f12);
		\draw (rev) -- (f13);
		\draw (f11) -- (compA);
		\draw (f12) -- (compB);
		\draw (f13) -- (compC);
		\draw (rev) -- (f21);
		\draw (rev) -- (f22);
		\draw (rev) -- (f23);
		\draw[arc] (f21) -- (salA);
		\draw[arc] (f22) -- (salB);
		\draw[arc] (f23) -- (salC);
		\draw (compA) -- (f31);
		\draw (compB) -- (f32);
		\draw (compC) -- (f33);
		\draw[arc] (f31) -- (salA);
		\draw[arc] (f32) -- (salB);
		\draw[arc] (f33) -- (salC);
	\end{pgfonlayer}
\end{tikzpicture}

%% file: files/ppcfg_split_example.tex
\begin{tikzpicture}
	\node[rv, minimum width = 6.2em, draw] (rev) {$Rev$};
	\node[rv, minimum width = 8.2em, draw, below right = 1.7em and 3.5em of rev, inner sep = 2.0pt] (comp) {$Comp(E)$};
	\node[rv, minimum width = 6.2em, draw, right = 12.0em of rev, inner sep = 2.0pt] (sal) {$Sal(E)$};
	\node[rv, minimum width = 8.2em, draw, above right = 1.7em and 3.5em of rev, inner sep = 2.0pt] (compA) {$Comp(alice)$};

	\pfsatd{$(rev)!0.5!(comp)$}{180}{$g_{1_{| C_1}}$}{f1a}{f1}{f1b}{1mm,yshift=-0.2em}
	\pfsatd{$(rev)!0.5!(sal)$}{270}{$g_{2_{| \top}}$}{f2a}{f2}{f2b}{0.1mm}
	\pfsatd{$(comp)!0.5!(sal)$}{0}{$g_{3_{| C_3}}$}{f3a}{f3}{f3b}{1mm,yshift=-0.2em}

	\pfsatd{$(rev)!0.5!(compA)$}{180}{$g'_{1_{| C'_1}}$}{f1pa}{f1p}{f1pb}{1mm,yshift=0.4em}
	\pfsatd{$(compA)!0.5!(sal)$}{0}{$g'_{3_{| C'_3}}$}{f3pa}{f3p}{f3pb}{1mm,yshift=0.4em}

	\begin{pgfonlayer}{bg}
		\draw (comp) -- (f1);
		\draw (f1) -- (rev);
		\draw (rev) -- (f2);
		\draw[arc] (f2) -- (sal);
		\draw (comp) -- (f3);
		\draw[arc] (f3) -- (sal);
		\draw (compA) -- (f1p);
		\draw (f1p) -- (rev);
		\draw (compA) -- (f3p);
		\draw[arc] (f3p) -- (sal);
	\end{pgfonlayer}
\end{tikzpicture}